\documentclass[a4paper,10pt]{article}

\usepackage{fullpage}

\usepackage[latin1]{inputenc}
\usepackage[T1]{fontenc}
\usepackage{graphicx} 
\usepackage{subfigure}

\usepackage{natbib}

\usepackage[ruled,algo2e]{algorithm2e}
\usepackage{hyperref}

\usepackage{amsmath}
\usepackage{amsthm}
\usepackage{amssymb}
\usepackage{amscd}

\usepackage{mathtools}
\mathtoolsset{showonlyrefs=true}

\newcommand{\M}{\mathcal{M}}
\newcommand{\g}{\mathcal{G}}
\newcommand{\e}{\mathcal{E}}
\newcommand{\R}{r}
\newcommand{\s}{\mathcal{S}}
\newcommand{\A}{\mathcal{A}}

\newcommand{\pol}{\mathcal{P}}
\newcommand{\E}{\mathbb{E}}

\newcommand{\C}[1]{\left\|\frac{d_{\mu,#1}}{\nu}\right\|_\infty}

\newtheorem{definition}{Definition}
\newtheorem{theorem}{Theorem}
\newtheorem{lemma}{Lemma}

\title{Policy Search: Any Local Optimum Enjoys \\ a Global Performance Guarantee}

\author{
Bruno Scherrer, \\
LORIA -- MAIA project-team,\\
Nancy, France,\\
\texttt{bruno.scherrer@inria.fr} \and
Matthieu Geist,\\
Sup\'elec -- IMS-MaLIS Research Group,\\
Metz, France,\\
\texttt{matthieu.geist@supelec.fr}
}

\begin{document}

\maketitle

\begin{abstract}%
Local Policy Search is a popular reinforcement learning approach for handling large state spaces. Formally, it searches locally in a parameterized policy space in order to maximize the associated value function averaged over some predefined distribution. It is probably commonly believed that the best one can hope in general from such an approach is to get a local optimum of this criterion. In this article, we show the following surprising result: \emph{any} (approximate) \emph{local optimum} enjoys a \emph{global performance guarantee}. We compare this guarantee with the one that is satisfied by Direct Policy Iteration, an approximate dynamic programming algorithm that does some form of Policy Search: if the approximation error of Local Policy Search may generally be bigger (because local search requires to consider a space of stochastic policies), we argue that the concentrability coefficient that appears in the performance bound is much nicer. Finally, we discuss several practical and theoretical consequences of our analysis.

\end{abstract}

\section{Introduction}
\label{sec:intro}

We consider the reinforcement learning problem formalized through Markov Decision Processes (MDP) \citep{Sutton:1998,Puterman:1994}, in the situation where the state space is large and approximation is required. On the one hand, Approximate Dynamic Programming (ADP) is a standard approach for handling large state spaces. It consists in mimicking in an approximate form the standard algorithms that were designed to optimize globally the policy (maximizing the associated value function for each state).
On the other hand, Local Policy Search (LPS) consists in parameterizing the policy (the so-called ``actor'') and locally maximizing the associated expected value function, for example using a (natural) gradient ascent~\citep{baxter:2001,Kakade:2001} (possibly with a critic~\citep{Sutton:1999b,Peters:2008}), expectation-maximization (EM)~\citep{kober_MACH_2011}, or even directly using some black-box optimization algorithm~\citep{Heidrich:2009}. 

The distinction we make here between ADP and LPS relies  on the overall algorithmic scheme that is considered (dynamic programming or local expected value maximization). For example, we see the Direct (or Classification-based) Policy Iteration (DPI) algorithm~\citep{Lagoudakis:2003a,Fern:2006,lazaric:2010} as belonging to the ADP family: even if it can be seen as a policy search method (since there is no representation for the value function), the algorithm follows the general approximate policy iteration (API) scheme. The Conservative Policy Iteration (CPI) algorithm~\citep{kakade:2002}---of which the analysis has close connections with what we are going to argue in this paper---might be considered at the frontier of ADP and LPS: it is based on a damped version of API, where each new policy is a convex mixture of the current policy and the greedy one, the precise combination being chosen such as guaranteeing an improvement in terms of the local fitness (the value function averaged over some predefined distribution).

Following the seminal works by \cite{Bertsekas:1996}, it has been shown that ADP algorithms enjoy global performance guarantees, bounding the loss of using the computed policy instead of using the optimal one as a function of the approximation errors involved along the iterations, for example for approximate policy iteration (API)~\citep{Munos03},  for approximate value iteration (AVI)~\citep{Munos_SIAM07}, or more generally for approximate modified policy iteration (AMPI)~\citep{scherrer:12_ampi}.
To the best of our knowledge, similar general guarantees do not exist in the literature for LPS algorithms. Bounds have been derived for the CPI algorithm by~\cite{kakade:2002}, and at first glance, one may think that this was due to its closeness to the ADP family. In general though, the best one can hope for LPS  is to get a local optimum of the optimized fitness (that is, a local maximum of the averaged value function), and the important question of the loss with respect to the optimal policy remains open. As for instance mentionned as the main ``future work'' in~\cite{Bhatnagar:2008}, where the convergence of a family of natural actor-critic algorithms is proven, ``\textit{[i]t is important to characterize the quality of converged solutions.''}

Experimentally, LPS methods seem to work pretty well. Applications to standard benchmarks \citep{baxter:2001,Kakade:2001,Peters:2008} and real applications such as robotics \citep{Peters:2008,kober_MACH_2011}  suggest that they are competitive with the ADP approach. Surprisingly, gradient-based and EM approaches, that are usually prone to be stuck in local optima, do not seem to be penalized in practice. Even more surprisingly, it was shown \citep{Kakade:2001} that a natural gradient ascent in the policy space can outperform ADP on the Tetris game.
The motivation of this paper is to fill the theoretical gap on the LPS methods and to explain to some extent their empirical successes.
Our main contribution (Theorem~\ref{th:mainresult}) is to show that \emph{any (approximate) local optimum of the expected value function enjoys a global performance guarantee}, similar to the one provided by ADP algorithms. The proof technique  we use is reminiscent of the one used for CPI, but our result is much more general and applies to a broad class of algorithms. Section~\ref{sec:background} provides the necessary background and states formally what we mean by local policy search. Section~\ref{sec:main} states and proves our main result. Section~\ref{sec:discussion} discusses it. Notably, a comparison to similar bounds for ADP is proposed and the practical consequences of the result are discussed. Section~\ref{sec:conclusion} opens some perspectives.

\section{Background and notations}
\label{sec:background}

Write $\Delta_X$ the set of probability distributions over a finite set $X$ and $Y^X$ the applications from $X$ to the set $Y$. By convention, all vectors are column vectors, except distributions which are row vectors (for left multiplication).
We consider a discounted MDP $\M = \{\s, \A, P, \R, \gamma\}$ \citep{Puterman:1994,Bertsekas:1995}, with $\s$ the finite state space\footnote{It is straightforward to extend our results to the case of infinite state space and compact action space. We chose the finite space setting for the ease and clarity of exposition.}, $\A$ the finite action space, $P\in(\Delta_\s)^{\s\times\A}$ the Markovian dynamics ($P(s'|s,a)$ denotes the probability of transiting to $s'$ from the $(s,a)$ couple), $\R\in\mathbb{R}^{\s\times \A}$ the bounded reward function and $\gamma\in[0,1)$ the discount factor. A stochastic policy $\pi\in(\Delta_\A)^\s$ associates to each state $s$ a probability distribution $\pi(.|s)$ over the action space $\A$.  For a given policy $\pi$, we write $\R_\pi\in\mathbb{R}^\s$ defined as $\R_\pi(s) = \sum_{a\in\A} \pi(a|s) \R(s,a) = \E_{a\sim \pi(.|s)}[\R(s,a)]$ and $P_\pi\in(\Delta_\s)^\s$ defined as $P_\pi(s'|s) = \sum_{a\in\A}\pi(a|s) P(s'|s,a) = \E_{a\sim\pi(.|s)}[P(s'|s,a)]$. The value function $v_\pi$ quantifies the quality of a policy $\pi$ for each state $s$ by measuring the expected cumulative reward received for starting in this state and then following the policy:
\begin{equation}
  v_\pi(s) = \E\left[\sum_{t\geq 0} \gamma^t \R_\pi(s_t)|s_0 = s, s_{t+1}\sim P_\pi(.|s_t)\right].
\end{equation}
The Bellman operator $T_\pi$ of policy $\pi$ associates to each function $v\in\mathbb{R}^\s$ the function defined as
\begin{equation}
  [T_\pi v](s) = \E\left[\R_\pi(s) + \gamma v(s')|s' \sim P_\pi(.|s)\right],
\end{equation}
or more compactly $T_\pi v = \R_\pi + \gamma P_\pi v$. The value function $v_\pi$ is the unique fixed point of $T_\pi$.

It is well known that there exists a policy $\pi_*$ that is optimal in the sense that it satisfies $v_{\pi_*}(s) \geq v_\pi(s)$ for all states $s$ and policies $\pi$. The value function $v_*$ is the unique fixed point of the following nonlinear Bellman equation:
\begin{equation}
  v_* = T v_* \text{ with } T v = \max_{\pi\in\A^\s} T_\pi v
\end{equation}
where the $\max$ is taken componentwise. Given any function $v \in \mathbb{R}^\s$, we say that a policy $\pi'$ is greedy with respect to $v$ if $T_{\pi'} v = T v$, and we write $\g(\pi)$ for the set of policies that are greedy with respect to the value $v_\pi$ of some policy $\pi$. The notions of optimal value function and greedy policies are fundamental to optimal
control because of the following property: any policy $\pi_*$ that is greedy with respect to the
  optimal value is an optimal policy and its value $v_{\pi_*}$ is equal to $v_*$. Therefore, another equivalent characterization is that $\pi$ is optimal if and only if it is greedy with respect to its value, that is if
\begin{equation}
\label{optimality}
\pi \in \g(\pi).
\end{equation}


For any distribution $\mu$, we define the $\gamma$-weighted occupancy measure\footnote{When it exists, this measure tends to the stationary distribution of $P_\pi$ when the discount factor tends to $1$.} induced by the policy $\pi$ when the initial state is sampled for $\mu$ as $d_{\mu,\pi} = (1-\gamma)\mu(I-\gamma P_\pi)^{-1} $ (we recall $\mu$ to be a row vector by convention) with $(I-\gamma P_\pi)^{-1} = \sum_{t\geq 0} (\gamma P_\pi)^t$. It can easily be seen that $\mu v_\pi=\frac{1}{1-\gamma} d_{\mu,\pi} \R_\pi$. For any two distributions $\mu$ and $\nu$, write $\left\|\frac{\mu}{\nu}\right\|_\infty$ the smallest constant $C$ satisfying $\mu(s) \leq C \nu(s)$, for any $s\in S$ (this constant is actually the supremum norm of the componentwise ratio, thus the notation).


From an algorithmic point of view,  Dynamic Programming methods compute the optimal value policy pair $(v_*, \pi_*)$ in an iterative way. When the problem is large and cannot be solved exactly, Approximate Dynamic Programming (ADP) refers to noisy implementations of these exact methods, where the noise is due to approximations at each iteration. For instance, Approximate Value and Policy Iteration respectively correspond to the following schemes:
\begin{equation}
  v_{k+1} = T v_k + \epsilon_k  \quad \text{ and } \quad \begin{cases}
    v_k = v_{\pi_k}  + \epsilon_k \\ \pi_{k+1} \in \g(v_k)
  \end{cases}.
\end{equation}
In the Local Policy Search (LPS) context on which we focus in this paper, we further need to specify the space where we are going to perform the search. We write $\Pi$ this set and assume that it is a \emph{convex} subset of $(\Delta_\A)^\s$. For a predefined distribution $\nu$ of interest, the problem addressed by LPS  can be cast as follows:
\begin{equation}
  \text{find } \pi \in \Pi \text{ such that } \pi \text{ is a local maximum of } J_\nu(\pi) = \E_{s\sim \nu}[v_\pi(s)].
\end{equation}
Assume that we are able to (approximately) find such a locally optiaml policy $\pi$. A natural question is: how close is $v_\pi$ to $v_* = v_{\pi_*}$? Quite surprisingly, and in contrast with most optimization problems, we will provide a generic answer to this question; this is the aim of the next section.


\section{Main result}
\label{sec:main}

In order to state our main result, we need to first define precisely what we mean by approximate local optimum.
For any pair of policies $\pi$ and $\pi'$ and coefficient $\alpha \in (0,1)$, we write $(1-\alpha)\pi + \alpha \pi'$ the stochastic mixture of $\pi$ and $\pi'$ with weights $(1-\alpha)$ and $\alpha$.
\begin{definition}[$\epsilon$-local optimum]
\label{def:local_opt}
  We say that a policy $\pi\in\Pi$ is an $\epsilon$-local optimum of $J_\nu(\pi)$ if:
  \begin{equation}
    \forall \pi' \in \Pi, \quad \lim_{\alpha\rightarrow 0} \frac{\nu v_{(1-\alpha)\pi + \alpha \pi'} - \nu v_\pi}{\alpha} \leq \epsilon.
  \end{equation}
\end{definition}
This condition is for instance satisfied when $\|\nabla_\pi J_\nu(\pi)\|_\infty \leq \epsilon$, it states that the gradient is ``small enough''. Notice that the assumption of a convex policy space is necessary for this definition.

Then, we define a relaxation of the set of policies that are greedy with respect to some given policy.
\begin{definition}[$\mu$-weighted $\epsilon$-greedy policies]
\label{def:relaxed_Bellman}
  We write $\g_\Pi(\pi,\mu,\epsilon)$ the set of policies which are $\epsilon$-greedy respectively to $\pi$ (in $\mu$-expectation). It is formally defined as
  \begin{equation}
    \g_\Pi(\pi,\mu,\epsilon) = \left\{ \pi' \in \Pi \text{ such that } \forall \pi'' \in \Pi, \; \mu T_{\pi'}v_\pi + \epsilon \geq \mu T_{\pi''}v_\pi\right\}.
  \end{equation}
\end{definition}
This is indeed a relaxation of $\g$, as it can be observed that for all policies $\pi$ and $\pi'$,
$$
\pi' \in \g(\pi)\quad\Leftrightarrow\quad\forall \mu\in\Delta_\s,~\pi' \in   \g_\Pi(\pi,\mu,0) \quad\Leftrightarrow\quad\exists \mu\in\Delta_\s,~ \mu>0,~ \pi' \in   \g_\Pi(\pi,\mu,0).
$$

We are now ready to state the first important result, which links the $\epsilon$-local optimality of Definition~\ref{def:local_opt} to some relaxed optimality characterization involving Definition~\ref{def:relaxed_Bellman}.
\begin{theorem}
  \label{th:localOpt_relaxedBellman}
  The policy $\pi\in\Pi$ is an $\epsilon$-local maximum of $J_\nu(\pi)$ if and only if it is $(1-\gamma)\epsilon$-greedy respectively to itself, in $d_{\nu,\pi}$-expectation:
  \begin{equation}
    \forall \pi' \in \Pi, \quad \lim_{\alpha\rightarrow 0} \frac{\nu v_{(1-\alpha)\pi + \alpha \pi'} - \nu v_\pi}{\alpha} \leq \epsilon
    \quad\quad \Leftrightarrow \quad\quad \pi \in \g_\Pi(\pi, d_{\nu,\pi},(1-\gamma)\epsilon).
  \end{equation}
\end{theorem}
The following technical (but simple) lemma will be useful for the proof.
\begin{lemma}
  \label{lemma:tech}
  For any policies $\pi$ and $\pi'$, we have
  \begin{equation}
    v_{\pi'} - v_{\pi} = (I-\gamma P_{\pi'})^{-1} (T_{\pi'}v_\pi - v_\pi).
  \end{equation}
\end{lemma}
\begin{proof}
  The proof uses the fact that the linear Bellman Equation $v_\pi=r_\pi + \gamma P_\pi v_\pi$ implies $v_\pi=(I-\gamma P_\pi)^{-1}r_\pi$. Then,
  \begin{align}
    v_{\pi'} - v_\pi &= (I-\gamma P_{\pi'})^{-1} \R_{\pi'} - v_\pi
     =  (I-\gamma P_{\pi'})^{-1}(\R_{\pi'} + \gamma P_{\pi'}v_\pi - v_{\pi})
    \\
    &= (I-\gamma P_{\pi'})^{-1}(T_{\pi'}v_\pi - v_{\pi}). \qedhere
  \end{align}
\end{proof}
\begin{proof}[Proof of Theorem~\ref{th:localOpt_relaxedBellman}]
  For any $\alpha$ and any $\pi'\in\Pi$, write $\pi_\alpha = (1-\alpha)\pi + \alpha \pi'$. Using Lemma~\ref{lemma:tech}, we have:
  \begin{equation}
    \nu (v_{\pi_\alpha} - v_\pi) = \nu (I-\gamma P_{\pi_\alpha})^{-1}(T_{\pi_\alpha}v_\pi - v_\pi).
  \end{equation}
  By observing that $r_{\pi_\alpha}=(1-\alpha)r_\pi + \alpha r_{\pi'}$ and $P_{\pi_\alpha}=(1-\alpha)P_\pi + \alpha P_{\pi'}$, it can be seen that $T_{\pi_\alpha} v_{\pi}=(1-\alpha)T_{\pi}v_\pi + \alpha T_{\pi'}v_\pi$. Thus, using the fact that $v_\pi=T_\pi v_\pi$,  we get:
  \begin{align}
    T_{\pi_\alpha} v_{\pi} - v_\pi &= (1-\alpha)T_{\pi}v_\pi + \alpha T_{\pi'} v_\pi - v_\pi \\
& = \alpha (T_{\pi'} v_{\pi} - v_\pi).
\end{align}
In parallel, we have
    \begin{align}
    (I-\gamma P_{\pi_\alpha})^{-1} &= (I - \gamma P_\pi + \alpha \gamma (P_{\pi'}-P_{\pi}))^{-1} \\
 &= (I-\gamma P_\pi)^{-1}(I+\alpha M),
  \end{align}
  where the exact form of the matrix $M$ does not matter. Put together, we obtain
  \begin{equation}
    \nu (v_{\pi_\alpha} - v_\pi) = \alpha \nu (I-\gamma P_\pi)^{-1} (T_{\pi'} v_{\pi} - v_\pi) + o(\alpha^2).
  \end{equation}
  Taking the limit, we obtain
  \begin{equation}
    \lim_{\alpha\rightarrow 0} \frac{\nu(v_{\pi_\alpha} - v_\pi)}{\alpha} = \nu (I-\gamma P_\pi)^{-1} (T_{\pi'} v_{\pi} - v_\pi) = (1-\gamma) d_{\nu,\pi} (T_{\pi'}v_\pi - v_\pi),
  \end{equation}
  from which the stated result follows.
\end{proof}

With Theorem~\ref{th:localOpt_relaxedBellman}, we know that if the LPS algorithm has produced a policy $\pi$ that is an $\epsilon$-local maximum, then it satisfies for some distribution $\mu$
\begin{align}
\pi \in \g_\Pi(\pi,\mu,\epsilon),
\label{relaxedoptimality}
\end{align}
that can be seen as a relaxed version of the original optimality characterization of Equation~\eqref{optimality}.
As we are about to see, in the Theorem to come next, such a relaxed optimality characterization can be shown to imply a global guarantee.
To state this result, we first need to define the ``$\nu$-greedy-complexity'' of our policy space, which measure how good $\Pi$ was designed so as to approximate the greedy operator, for a starting distribution $\nu$.
\begin{definition}[$\nu$-greedy-complexity]
  \label{def:policySpaceComplexit}
  We define $\e_\nu(\Pi)$ the $\nu$-greedy-complexity  of the policy space $\Pi$ as
  \begin{equation}
    \e_\nu(\Pi) = \max_{\pi \in \Pi} \min_{\pi'\in\Pi}\left(d_{\nu,\pi} \left(T v_\pi - T_{\pi'} v_\pi\right)\right).
  \end{equation}
\end{definition}
It is clear that if $\Pi$ contains any deterministic policy (a strong assumption), then $\e_\nu(\Pi)=0$.
Given this definition, we are ready to state our second important result.
\begin{theorem}
  \label{th:relaxedBellman_globalOpt}
  If $\pi \in \g_\Pi(\pi,d_{\nu,\pi},\epsilon)$, then for any policy $\pi'$ and for any distribution $\mu$ over $\s$, we have
  \begin{equation}
    \mu v_{\pi'} \leq \mu v_{\pi} + \frac{1}{(1-\gamma)^2} \C{\pi'}(\e_\nu(\Pi) + \epsilon).
  \end{equation}
\end{theorem}
Notice that this theorem is actually a slight\footnote{Theorem~\ref{th:relaxedBellman_globalOpt} holds for any policy $\pi'$, not only for the optimal one, and the error term is split up (which is necessary to provide a more general result).} generalization of Theorem~6.2 of \cite{kakade:2002}. We provide the proof, that is essentially the same as that given in \cite{kakade:2002}, for the sake of completeness.
\begin{proof}
  Using again Lemma~\ref{lemma:tech} and the fact that $T v_\pi \ge T_{\pi'} v_\pi$, we have
  \begin{align}
    \mu (v_{\pi'}-v_\pi) &= \mu(I-\gamma P_{\pi'})^{-1}(T_{\pi'}v_\pi - v_\pi)
    = \frac{1}{1-\gamma} d_{\mu,\pi'} (T_{\pi'}v_\pi - v_\pi)
    \leq \frac{1}{1-\gamma} d_{\mu,\pi'}  (T v_\pi - v_\pi).
  \end{align}
  Since $T v_\pi - v_\pi \ge 0$ and $d_{\nu,\pi} \geq (1-\gamma)\nu$, we get
  \begin{align}
    \mu (v_{\pi'}-v_\pi) &\leq \frac{1}{1-\gamma} \C{\pi'} \nu (T v_\pi - v_\pi)
    \leq \frac{1}{(1-\gamma)^2}  \C{\pi'}  d_{\nu,\pi} (T v_\pi - v_\pi).
  \end{align}
  Finally, using $d_{\nu,\pi} (T v_\pi - v_\pi)= (d_{\nu,\pi}  T v_\pi - d_{\nu,\pi}  v_\pi)$, we obtain
  \begin{align}
    \mu (v_{\pi'}-v_\pi)
    &\leq \frac{1}{(1-\gamma)^2} \C{\pi'} (d_{\nu,\pi}  T v_\pi - \max_{\pi'\in \Pi} d_{\nu,\pi} T_{\pi'}v_\pi + \max_{\pi'\in \Pi} d_{\nu,\pi} T_{\pi'}v_\pi - d_{\nu,\pi}  v_\pi)
    \\
    &\leq \frac{1}{(1-\gamma)^2} \C{\pi'} (\e_\nu(\Pi) + \epsilon) \qedhere
  \end{align}
\end{proof}

The main result of the paper is a straightforward combination of the results of both Theorems.
\begin{theorem}
  \label{th:mainresult}
  Assume that the policy $\pi$ is an $\epsilon$-local optimum of $J_\nu(\pi)$ over $\Pi$, that is
  \begin{equation}
    \forall \pi' \in \Pi, \quad \lim_{\alpha\rightarrow 0} \frac{J_\nu(\pi_\alpha) - J_\nu(\pi)}{\alpha} \leq \epsilon \quad \text{ (with $\pi_\alpha = (1-\alpha)\pi + \alpha \pi'$)},
  \end{equation}
  then, $\pi$ enjoys the following global performance guarantee:
  \begin{equation}
    0\leq \E_{s\sim\mu}[v_*(s)-v_\pi(s)] \leq \frac{1}{1-\gamma} \C{\pi_*} \left(\frac{\e_\nu(\Pi)}{1-\gamma} + \epsilon \right).
  \end{equation}
\end{theorem}

\section{Discussion}
\label{sec:discussion}

We have just shown that \emph{any policy search algorithm} that is able to estimate any $\epsilon$-close local optimum of the fitness function $J_\nu(\pi) = \E_{s\sim\nu}[v_\pi(s)]$ \emph{actually comes with a global performance guarantee}. In this section, we discuss the relations of our analyses with previous works, we compare this guarantee with the standard ones of approximate dynamic programming (focusing particularly on the DPI algorithm) and we discuss some practical and theoretical consequences of our analysis.

\subsection{Closely related analysis}
\label{subsec:relatedResults}

Though the main result of our paper is Theorem~\ref{th:mainresult}, and
since Theorem~\ref{th:relaxedBellman_globalOpt} appears in a very
close form in \cite{kakade:2002}, our main technical contribution is
Theorem~\ref{th:localOpt_relaxedBellman} that highlights a deep
connection between local optimality and a relaxed Bellman optimality
characterization. A result, that is similar in flavor, is derived by
\cite{Kakade:2001} for the Natural Policy Gradient algorithm: Theorem
3 there shows that natural gradient updates are moving the policy
towards the solution of a (DP) update. The author even writes: ``The
natural gradient could be efficient far from the maximum, in that it
is pushing the policy toward choosing greedy optimal actions''. Though
there is an obvious connection with our work, the result there is
limited since 1) it seems to be specific to the natural gradient
approach (though our result is very general), and 2) it is not
exploited so as to connect with a global performance guarantee.

A performance guarantee very similar to the one we provide in
Theorem~\ref{th:relaxedBellman_globalOpt} was first derived for CPI in
\cite{kakade:2002}. The main difference is that the term $(\e_\nu(\Pi)
+ \epsilon)$ of Theorem~\ref{th:relaxedBellman_globalOpt} is replaced
there by some global precision $\epsilon$, that corresponds to the
error made by a classifier that is used as an approximate greedy policy
chooser.  Similarly to the work we have just mentioned on the Natural
Policy Gradient, this result of the literature was certainly considered specific to the
CPI algorithm, that has unfortunately not been used widely in practice
probably because of its somewhat complex
implementation. In contrast, we show in this paper that such a performance
guarantee is valid for any method that finds a policy that satisfies a
relaxed Bellman identity like that given
Equation~\eqref{relaxedoptimality}, among which one finds many widely
used methods that do Local Policy Search.

\subsection{Relations to bounds of approximate dynamic programming}

The performance guarantee of any approximate dynamic programming algorithm implies \emph{(i)} a (quadratic) dependency on the average horizon $\frac{1}{1-\gamma}$, \emph{(ii)} a concentration coefficient (which quantifies the divergence between the worst discounted average future state distribution when starting from the measure of interest, and the distribution used to control the estimation errors), and \emph{(iii)} an error term linked to the estimation error encountered at each iteration (which can be due to the approximation of value functions and/or policies). Depending on what quantity is estimated, a comparison of these estimation errors may be hard. To ease the comparison, the following discussion focuses on the Direct Policy Iteration algorithm that does some form of policy search. Note however that several aspects of our comparison holds for other ADP algorithms.

Direct Policy Iteration (DPI) \citep{Lagoudakis:2003a,lazaric:2010} is an approximate policy iteration algorithm where at each iteration, $(i)$ the value function is estimated for a set of states using Monte Carlo rollouts and $(ii)$ the greedy policy (respectively to the current value function) is approximately chosen in some predefined policy space (through a weighted classification problem). Write $\pol$ this policy space, which is typically a set of \emph{deterministic policies}. For an initial policy $\pi_0$ and a given distribution $\nu$, the DPI algorithms iterates as follows:
\begin{equation}
  \text{pick }\pi_{k+1}\in\pol \text{ such as (approximately) minimizing  } \nu(T v_{\pi_k} - T_{\pi_{k+1}} v_{\pi_k}).
\end{equation}
To provide the DPI bound, we need an alternative concentration coefficient as well as some new error characterizing $\pol$.
Let $C_{\mu,\nu}^*$ be the concentration coefficient defined as
\begin{equation}
  C_{\mu,\nu}^* = (1-\gamma)^2 \sum_{i=0}^\infty \sum_{j=0}^\infty \gamma^{i+j} \sup_{\pi\in\A^\s} \left\|\frac{\mu(P_{\pi_*})^i (P_\pi)^j}{\nu}\right\|_\infty.
\end{equation}
We need also a measure of the complexity of the policy space $\pol$, similar to $\e_\nu$:
\begin{equation}
  \e'_\nu(\pol) = \max_{\pi\in\pol}\min_{\pi'\in\pol}(\nu(T v_\pi - T_{\pi'}v_\pi))
\end{equation}
Let also $e$ be an estimation error term, which depends on the number of samples and which tends to zero as the number of samples tends to infinity (at a rate depending on the chosen classifier).
The performance guarantee of DPI \citep{lazaric:2010,Ghavamzadeh:12} can be expressed as follows:
\begin{equation}
  \limsup_{k\rightarrow\infty}\mu(v^* - v_{\pi_k}) \leq \frac{C_{\mu,\nu}^*}{(1-\gamma)^2}(\e'_\nu(\pol) + e).
\end{equation}
This bound is to be compared with the result of Theorem~\ref{th:mainresult}, regarding the three terms involved: the average horizon, the concentration coefficient and the greedy error term. Each term is discussed now, a brief summary being provided in Table~\ref{tab:bound}. As said in Section~\ref{subsec:relatedResults}, the LPS bound is really similar to the CPI one, and the CPI and DPI bounds have been extensively compared in \citep{Ghavamzadeh:12}. Our discussion can be seen as complementary.

\begin{table}
  \begin{center}
  \caption{Comparison of the performance guarantees for LPS and DPI}
  \label{tab:bound}
  \begin{tabular}{|l||c|c|c|c|}
    \hline
     & bounded term & horizon term & concentration term & error term
    \\ \hline \hline
    LPS & $\mu(v_* - v_\pi)$ & $\frac{1}{(1-\gamma)^2}$ & $\left\|\frac{d_{\mu,\pi_*}}{\nu}\right\|_\infty$ & $\e_\nu(\Pi) + \epsilon (1-\gamma)$
    \\ \hline
    DPI & ${\small \limsup_{k\rightarrow\infty}\mu(v^* - v_{\pi_k})}$ &  $\frac{1}{(1-\gamma)^2}$ & $C_{\mu,\nu}^*$ & $\e'_\nu(\pol) + e$
    \\
    \hline
  \end{tabular}
  \end{center}
\end{table}


\textbf{Horizon term.} Both bounds have a quadratic dependency on the average horizon $\frac{1}{1-\gamma}$. For approximate dynamic programming, this bound can be shown to be tight \citep{scherrer:12_ampi}, the only known solution to improve this being to introduce non-stationary policies \citep{scherrer:2012_nips}. The tightness of this bound for policy search is an open question. However, we suggest later in Section~\ref{sec:betterlearning} a possible way to improve this.

\textbf{Concentration coefficients.} Both bounds involve a concentration coefficient. Even if they are different, they can be linked.
\begin{theorem}
  \label{th:concentration}
  We always have that:
  $
  \C{\pi_*} \leq \frac{1}{1-\gamma} C_{\mu,\nu}^*.
$
  Moreover, if there always exists a $\nu$ such that $\C{\pi_*}< \infty$ (with $\nu = d_{\mu,\pi_*}$), there might not exist a $\nu$ such that $C_{\mu,\nu}^*<\infty$.
\end{theorem}
\begin{proof}
Let us first consider the inequality. By using the definition of $d_{\mu,\pi_*}$ and eventually the fact that $d_{\mu,\pi_*} \geq (1-\gamma) \nu$, we have
  \begin{align}
    C_{\mu,\nu}^* &= (1-\gamma)^2 \sum_{i=0}^\infty \sum_{j=0}^\infty \gamma^{i+j} \sup_{\pi\in(\Delta_\A)^\s} \left\|\frac{\mu(P_{\pi_*})^i (P_\pi)^j}{\nu}\right\|_\infty
    \\
    &\geq (1-\gamma)^2 \left\|\sum_{i,j=0}^\infty \gamma^{i+j} \frac{\mu (P_{\pi_*})^{i+j}}{\nu}\right\|_\infty
    = (1-\gamma) \left\|\sum_{i=0}^\infty \gamma^{i} \frac{d_{\mu,\pi_*} (P_{\pi_*})^{i}}{\nu}\right\|_\infty
    \\
    &\geq (1-\gamma)^2 \left\|\sum_{i=0}^\infty \gamma^{i} \frac{\mu(P_{\pi_*})^{i}}{\nu}\right\|_\infty
    = (1-\gamma) \C{\pi_*}.
  \end{align}

  For the second part of the theorem, consider an MDP with $N$ states and $N$ actions, with $\mu = \delta_1$ being a dirac on the first state, and such that from here action $a\in[1;N]$ leads in state $a$ deterministically. Write $c = \sup_{\pi\in\A^\s} \|\frac{\mu P_\pi}{\nu}\|_\infty$ the first term defining $C_{\mu,\nu}^*$. For any $\pi$, we have $\mu P_\pi\leq c\nu$. Thus, for any action $a$ we have $\delta_a \leq c \nu \Rightarrow 1 \leq c \nu(a)$. Consequently, $1 = \sum_{i=1}^N \nu(i) \geq \frac{1}{c} \sum_{i=1}^N 1 \Leftrightarrow c \geq N$. This being true for arbitrary $N\in\mathbb{N}$, we get $c=\infty$ and thus $C_{\mu,\nu}^* = \infty$.
\end{proof}
The first part of this result was already stated in \cite{Ghavamzadeh:12}, for the comparison of CPI (which involves the same concentration as LPS) and DPI. The second part is new: it tells that we may have $\C{\pi_*} \ll C_{\mu,\nu}^*$, which is clearly in favor of LPS (and CPI, as a side effect).

\textbf{Error terms.} Both bounds involve an error term. The terms $\epsilon$ (LPS) and $e$ (DPI) can be made arbitrarily small by increasing the computational effort (the time devoted to run the algorithm and the amount of samples used), though not much more can be said without studying a specific algorithmic instance (\textit{e.g.}, type of local search for LPS or type of classifier for DPI). The terms defining the ``greedy complexity'' of policy spaces can be more easily compared. Because they use different distributions that can be compared ($d_{\nu,\pi} \ge (1-\gamma) \nu$), we have for all policy spaces $\Pi$,
\begin{equation}
  \e'_{\nu}(\Pi) \leq \frac{\e_\nu(\Pi)}{1-\gamma}.
\end{equation}
However, this result (already stated in \cite{Ghavamzadeh:12}) does not take into account the fact that LPS (or CPI for the discussion of \cite{Ghavamzadeh:12}) works with \emph{stochastic policies} while DPI works with \emph{deterministic policies}. For example, assume that $\Pi$ is the convex closure of $\pol$. In this case, we would have $\e'_{\nu}(\pol) \leq \e'_{\nu}(\Pi)$. Therefore, this error term would be more in favor of DPI than LPS: the search space is presumably smaller (while possibly allowing to represent the same deterministic greedy policies).

%

\subsection{Practical and theoretical consequences of our analysis}

Finally, this section provides a few important consequences of our analysis and of Theorem~\ref{th:mainresult} in particular.

\textbf{Rich policy and equivalence between local and global optimality.}
If the the policy space is very rich, one can easily show that any local optimum is actually global (this result being a direct corollary of Theorem~\ref{th:mainresult}).
\begin{theorem}
  \label{cor:optGlob}
  Let $\nu>0$ be a distribution. Assume that the policy space is rich in the sense that $\e_\nu(\Pi)=0$, and that $\pi$ is an (exact) local optimum of $J_\nu$ ($\epsilon = 0$). Then, we have $v_\pi = v_*$.
\end{theorem}
If this result is well-known in the case of tabular policies, it is to our knowledge new in such a general case (acknowledging that $\e_\nu(\Pi)=0$ is a rather strong assumption).


\textbf{Choice of the sampling distribution.}
Provided the result of Theorem~\ref{th:mainresult}, and as also mentionned about CPI in \cite{kakade:2002} since it satisfies a similar bound, if one wants to optimize the policy according to a distribution $\mu$ (that is, such that $\mu(v_* - v_\pi)$ is small), then one should optimize the fitness $J_\nu$ with the distribution $\nu \simeq d_{\mu,\pi_*}$ (so as to minimize the coefficient $\C{\pi_*}$). Ideally, one should sample states based on trajectories following the optimal policy $\pi_*$ starting from states drawn according to $\mu$ (which is not surprising). This is in general not realistic since we do not know the optimal policy $\pi_*$,
but practical solutions may be envisioned.

First, this means that one should sample states in the ``interesting'' part of the state space, that is where the optimal policy is believed to lead. This is a  natural piece of information that a domain expert should be able to provide and this is in general much easier than actually controlling the system with the optimal policy (or with a policy that  leads to these interesting parts of the state space). Also, though we leave the precise study of this idea for future research, a natural practical approach for setting the distribution $\nu$ would be to compute a sequence of policies $\pi_1,\pi_2,\dots$ such that for all $i$, $\pi_i$ is a local optimum of $\pi \mapsto J_{d_{\nu,{\pi_{i-1}}}}(\pi)$, that is of the criterion weighted by the region visited by the previous policy $\pi_{i-1}$. It may particularly be interesting to study whether the convergence of such an iterative process leads to interesting guarantees.

One may also notice that Theorem~\ref{th:mainresult} may be straightforwardly written more generally for any policy. If $\pi$ is an $\epsilon$-local optimum of $J_\nu$ over $\Pi$, then for any stochastic policy $\pi'$ we have
\begin{equation}
  \mu v_{\pi'} \leq \mu v_\pi + \frac{1}{1-\gamma} \C{\pi'} \left(\frac{\e_\nu(\Pi)}{1-\gamma} + \epsilon\right).
\end{equation}
Therefore, one can sample trajectories according to an acceptable (and known) controller $\pi'$ so as to get state samples to optimize $J_{d_{\nu,\pi'}}$. More generally, if we know where a good policy $\pi'$ leads the system to from some initial distribution $\mu$, we can learn a policy $\pi$ that is guaranteed to be approximately as good (and potentially better).


\textbf{A better learning problem?}
\label{sec:betterlearning}
With the result of Theorem~\ref{th:mainresult}, we have a squared dependency of the bound on the effective average horizon $\frac{1}{1-\gamma}$. For approximate dynamic programming, it is known that this dependency is tight \citep{scherrer:2012_nips}. At the current time, this is an open question for policy search. However, we can somehow improve the bound. We have shown that the $\epsilon$-local optimality of a policy $\pi$ implies that it satisfies a relaxed Bellman global optimality characterization, $\pi \in \g_\Pi(\pi,d_{\nu,\pi},\epsilon)$, which in turns implies Theorem~\ref{th:mainresult}. The following result, involving a slightly simpler relaxed Bellman equation, can be proved similarly to Theorem~\ref{th:relaxedBellman_globalOpt}:
\begin{equation}
  \pi\in\g_\Pi(\pi,\nu,\epsilon) \quad \Leftrightarrow \quad  \mu v_{\pi'} \leq \mu v_\pi + \frac{1}{1-\gamma} \C{\pi'}(\e_\nu(\Pi)+\epsilon).
\end{equation}
A policy satisfying the left hand side would have an improved dependency on the horizon ($\frac{1}{1-\gamma}$ instead of $\frac{1}{(1-\gamma)^2}$). At the current time, we do not know whether there exists an efficient algorithm for computing a policy satisfying $\pi\in\g_\Pi(\pi,\nu,\epsilon)$. The above guarantee suggests that solving such a problem may improve over traditional policy search approaches.

\section{Conclusion}
\label{sec:conclusion}

In the past years, local policy search algorithms have been shown to be practical viable alternatives to the more traditional approximate dynamic programming field. The derivation of global performance guarantees for such approaches, probably considered as a desperate case, was to our knowledge never considered in the literature. In this article, we have shown a surprising result: \emph{any Local Policy Search algorithm}, as long as it is able to \emph{provide an approximate local optimum} of $J_{\nu}(\pi)$, \emph{enjoys a global performance guarantee} similar to the ones of approximate dynamic programming algorithms.
From a theoretical viewpoint, there is thus no reason to prefer approximate dynamic programming over policy search (practical reasons -- \textit{e.g.}, necessity of a simulator -- or other theoretical reasons -- \textit{e.g.}, rate of convergence -- may come in line).

Since the bounds of ADP are known to be tight,
the question whether the guarantee we have provided is tight  constitutes an interesting future research direction.
We suggested that it may be a better learning strategy to look for a policy $\pi$ satisfying $\pi\in\g_\Pi(\pi,\nu,\epsilon)$ instead of searching for a local maximum of $J_\nu$, as it leads to a better bound. Designing an algorithm that would do so efficiently is another interesting perspective. Finally, we here only considered pure actor algorithms, with only a parameterization of the policy. The extension of  our analysis to situations where one also uses a critic (a parameterization of the value function) is a natural track to explore.

\bibliographystyle{natbib2}
\bibliography{biblio.bib} 

\begin{thebibliography}{}

\bibitem[Baxter and Bartlett(2001)Baxter and Bartlett]{baxter:2001}
Baxter, J. and Bartlett, P.~L. (2001).
\newblock Infinite-horizon gradient-based policy search.
\newblock {\em Journal of Artificial Intelligence Research\/}, {\bf 15},
  319--350.

\bibitem[Bertsekas and Tsitsiklis(1996)Bertsekas and
  Tsitsiklis]{Bertsekas:1996}
Bertsekas, D. and Tsitsiklis, J. (1996).
\newblock {\em Neuro-Dynamic Programming\/}.
\newblock Athena Scientific.

\bibitem[Bertsekas(1995)Bertsekas]{Bertsekas:1995}
Bertsekas, D.~P. (1995).
\newblock {\em {Dynamic Programming and Optimal Control}\/}.
\newblock Athena Scientific.

\bibitem[Bhatnagar {\em et~al.}(2007)Bhatnagar, Sutton, Ghavamzadeh, and
  Lee]{Bhatnagar:2008}
Bhatnagar, S., Sutton, R.~S., Ghavamzadeh, M., and Lee, M. (2007).
\newblock Incremental natural actor-critic algorithms.
\newblock In {\em Conference on Neural Information Processing Systems
  (NIPS)\/}, Vancouver, Canada.

\bibitem[Fern {\em et~al.}(2006)Fern, Yoon, and Givan]{Fern:2006}
Fern, A., Yoon, S., and Givan, R. (2006).
\newblock {Approximate Policy Iteration with a Policy Language Bias: Solving
  Relational Markov Decision Processes}.
\newblock {\em Journal of Artificial Intelligence Research\/}, {\bf 25},
  75--118.

\bibitem[Ghavamzadeh and Lazaric(2012)Ghavamzadeh and Lazaric]{Ghavamzadeh:12}
Ghavamzadeh, M. and Lazaric, A. (2012).
\newblock {Conservative and Greedy Approaches to Classification-based Policy
  Iteration}.
\newblock In {\em Conference on Artificial Intelligence (AAAI)\/}.

\bibitem[Heidrich-Meisner and Igel(2008)Heidrich-Meisner and
  Igel]{Heidrich:2009}
Heidrich-Meisner, V. and Igel, C. (2008).
\newblock Evolution strategies for direct policy search.
\newblock In {\em Proceedings of the 10th international conference on Parallel
  Problem Solving from Nature: PPSN X\/}, pages 428--437.

\bibitem[Kakade(2001)Kakade]{Kakade:2001}
Kakade, S. (2001).
\newblock {A Natural Policy Gradient}.
\newblock In {\em Neural Information Processing Systems (NIPS)\/}, pages
  1531--1538.

\bibitem[Kakade and Langford(2002)Kakade and Langford]{kakade:2002}
Kakade, S. and Langford, J. (2002).
\newblock Approximately optimal approximate reinforcement learning.
\newblock In {\em International Conference on Machine Learning\/}.

\bibitem[Kober and Peters(2011)Kober and Peters]{kober_MACH_2011}
Kober, J. and Peters, J. (2011).
\newblock {Policy Search for Motor Primitives in Robotics}.
\newblock pages 171--203.

\bibitem[Lagoudakis and Parr(2003)Lagoudakis and Parr]{Lagoudakis:2003a}
Lagoudakis, M.~G. and Parr, R. (2003).
\newblock Reinforcement learning as classification: Leveraging modern
  classifiers.
\newblock In {\em International Conference on Machine Learning\/}, pages
  424--431.

\bibitem[Lazaric {\em et~al.}(2010)Lazaric, Ghavamzadeh, and
  Munos]{lazaric:2010}
Lazaric, A., Ghavamzadeh, M., and Munos, R. (2010).
\newblock Analysis of a classification-based policy iteration algorithm.
\newblock In {\em International Conference on Machine Learning\/}, pages
  607--614.

\bibitem[Munos(2003)Munos]{Munos03}
Munos, R. (2003).
\newblock Error bounds for approximate policy iteration.
\newblock In {\em International Conference on Machine Learning\/}, pages
  560--567.

\bibitem[Munos(2007)Munos]{Munos_SIAM07}
Munos, R. (2007).
\newblock Performance bounds in {Lp} norm for approximate value iteration.
\newblock {\em SIAM J. Control and Optimization\/}.

\bibitem[Peters and Schaal(2008)Peters and Schaal]{Peters:2008}
Peters, J. and Schaal, S. (2008).
\newblock {Natural Actor-Critic}.
\newblock {\em Neurocomputing\/}, {\bf 71}, 1180--1190.

\bibitem[Puterman(1994)Puterman]{Puterman:1994}
Puterman, M.~L. (1994).
\newblock {\em Markov Decision Processes: Discrete Stochastic Dynamic
  Programming\/}.
\newblock Wiley-Interscience.

\bibitem[Scherrer and Lesner(2012)Scherrer and Lesner]{scherrer:2012_nips}
Scherrer, B. and Lesner, B. (2012).
\newblock {On the Use of Non-Stationary Policies for Stationary
  Infinite-Horizon Markov Decision Processes}.
\newblock In P.~Bartlett, F.~Pereira, C.~Burges, L.~Bottou, and K.~Weinberger,
  editors, {\em Advances in Neural Information Processing Systems 25\/}, pages
  1835--1843.

\bibitem[Scherrer {\em et~al.}(2012)Scherrer, Gabillon, Ghavamzadeh, and
  Geist]{scherrer:12_ampi}
Scherrer, B., Gabillon, V., Ghavamzadeh, M., and Geist, M. (2012).
\newblock {Approximate Modified Policy Iteration}.
\newblock In {\em {International Conference on Machine Learning (ICML)}\/}.

\bibitem[Sutton and Barto(1998)Sutton and Barto]{Sutton:1998}
Sutton, R. and Barto, A. (1998).
\newblock {\em {Reinforcement Learning, An introduction}\/}.
\newblock BradFord Book. The MIT Press.

\bibitem[Sutton {\em et~al.}(1999)Sutton, McAllester, Singh, and
  Mansour]{Sutton:1999b}
Sutton, R.~S., McAllester, D.~A., Singh, S.~P., and Mansour, Y. (1999).
\newblock {Policy Gradient Methods for Reinforcement Learning with Function
  Approximation}.
\newblock In {\em Neural Information Processing Systems (NIPS)\/}, pages
  1057--1063.

\end{thebibliography}

\end{document}